\documentclass[10pt,twocolumn,letterpaper]{article}

\usepackage{cvpr}              % To produce the CAMERA-READY version
\usepackage{graphicx}
\usepackage{amsmath}
\usepackage{amssymb}
\usepackage{booktabs}

\usepackage[pagebackref,breaklinks,colorlinks]{hyperref}

\usepackage[capitalize]{cleveref}
\crefname{section}{Sec.}{Secs.}
\Crefname{section}{Section}{Sections}
\Crefname{table}{Table}{Tables}
\crefname{table}{Tab.}{Tabs.}

\def\cvprPaperID{3279} % *** Enter the CVPR Paper ID here
\def\confName{CVPR}
\def\confYear{2022}

\usepackage{overpic}
\usepackage{enumitem} %< control spacing in itemize/enumerate/...
\usepackage{overpic} %< add raw math symbols to figures
\usepackage{color}
\definecolor{turquoise}{cmyk}{0.65,0,0.1,0.3}
\definecolor{purple}{rgb}{0.65,0,0.65}
\definecolor{dark_green}{rgb}{0, 0.5, 0}
\definecolor{orange}{rgb}{0.8, 0.6, 0.2}
\definecolor{red}{rgb}{0.8, 0.2, 0.2}
\definecolor{darkred}{rgb}{0.6, 0.1, 0.05}
\definecolor{blueish}{rgb}{0.0, 0.3, .6}
\definecolor{light_gray}{rgb}{0.7, 0.7, .7}
\definecolor{pink}{rgb}{1, 0, 1}
\definecolor{greyblue}{rgb}{0.25, 0.25, 1}

\usepackage{blindtext}

\renewcommand{\paragraph}[1]{\vspace{1em}\noindent\textbf{#1}.}
\usepackage{amsthm}
\newtheorem{lemma}{Lemma}
\newtheorem{theorem}{Theorem}

\usepackage{algorithm}
\usepackage{algpseudocode}
\usepackage{xcolor}

\usepackage{float}
\begin{document}
\title{Hierarchical Nearest Neighbor Graph Embedding for Efficient Dimensionality Reduction}

\author{M. Saquib Sarfraz  $^{*1,2}$,  Marios Koulakis $^{*1}$, Constantin Seibold$^{1}$, Rainer Stiefelhagen$^{1}$ \\
\small{$^{1}$ Karlsruhe Institute of Technology,
$^{2}$ Mercedes-Benz Tech Innovation} \\
\small{$^{*}$ Equal contribution
}
% For a paper whose authors are all at the same institution,
% omit the following lines up until the closing ``}''.
% Additional authors and addresses can be added with ``\and'',
% just like the second author.
% To save space, use either the email address or home page, not both
% \and
% Second Author\\
% Institution2\\
% First line of institution2 address\\
% {\tt\small secondauthor@i2.org}
}
\maketitle
\begin{abstract}
Dimensionality reduction is crucial both for visualization and preprocessing high dimensional data for machine learning. We introduce a novel method based on a hierarchy built on 1-nearest neighbor graphs in the original space which is used to preserve the grouping properties of the data distribution on multiple levels. The core of the proposal is an optimization-free projection that is competitive with the latest versions of t-SNE and UMAP in performance and visualization quality while being an order of magnitude faster in run-time. 
Furthermore, its interpretable mechanics, the ability to project new data, and the natural separation of data clusters in visualizations make it a general purpose unsupervised dimension reduction technique.
In the paper, we argue about the soundness of the proposed method and evaluate it on a diverse collection of datasets with sizes varying from 1K to 11M samples and dimensions from 28 to 16K. 
We perform comparisons with other state-of-the-art methods on multiple metrics and target dimensions highlighting its efficiency and performance. Code is available at \url{https://github.com/koulakis/h-nne}
\end{abstract}

\section{Introduction}
\label{sec:intro}
\setlength{\belowcaptionskip}{-16pt}
\setlength{\abovecaptionskip}{-4pt}

\begin{figure}[t]
\begin{center}
% \begin{overpic} 
% [width=\linewidth]
% {example-image-a}
% \end{overpic}
\includegraphics[width=\linewidth]{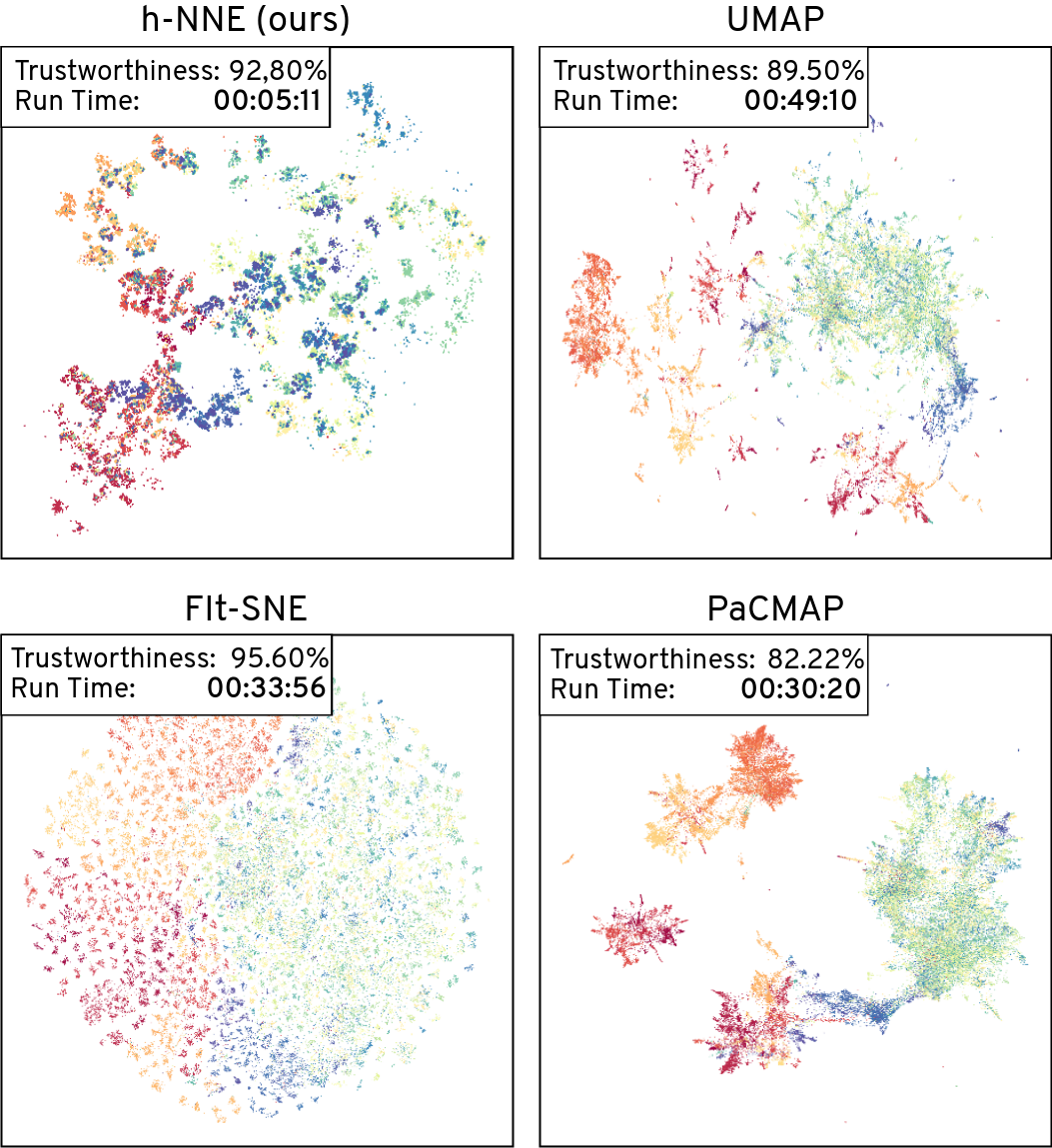}
\end{center}
\caption{
Visualization of the entire ImageNet dataset. Speed and embedding quality of dimension reduction methods.
}
\label{fig:teaser}
\end{figure}
Dimensionality reduction techniques are now increasingly used in many fields of science and have to cope with an ever increasing size of real-world datasets. 
It plays an important role both for visualization and processing of high dimensional data.  
Much of current research is focused on finding unsupervised algorithms that are both scalable to massive data and are able to preserve the structure of data in less dimensions. Most of them attempt to retain the local or global structure of the data by optimizing over pairwise distances in the target space. 
Two main directions for the current dimension reduction techniques can be identified with respect to how such local or global neighborhood is preserved in terms of the distances. Methods such as PCA~\cite{hotelling1933analysis}, MDS~\cite{kruskal1964multidimensional} and Sammom mapping~\cite{sammon1969nonlinear} try to preserve the global distances among all samples in the data. Whereas more recent popular methods such as t-SNE~\cite{tsne,bh_tsne}, LargeVis~\cite{tang2016visualizing}, and UMAP~\cite{umap,pumap} seek to additionally preserve the local structure e.g. by preserving the distance relations in the k-neighborhood of each data sample. To retain such relations, these methods generally have to solve an optimization problem with the goal of matching the distribution of distances in the target space with their distribution in the original space. For instance, t-SNE minimizes the Kullback-Leibler divergence between distributions of k-nearest neighbor (k-NN) distances fitted in the high and low-dimensional space. Similarly, the more recent method UMAP optimizes the embedding in the target space with the goal of preserving the 1-skeleton of fuzzy simplicial sets constructed in the original space. Such optimizations are computationally expensive in nature and account for the main complexity of these algorithms, thus limiting their run-time performance on large scale datasets. 

In this paper, we present a different approach which instead of relying on point-level optimization, captures multi-stage NN properties of the data and, using those, projects points in a simple algorithmic way. The main tools used to build this structure are Nearest Neighbor Graphs (NNGs) which have been well studied. In~\cite{eppstein1997nearest} Epstein~\etal show that for a 1-NNG, its NN relations are well preserved in a low dimensional space when the edges are placed on a monotone logical grid~\cite{boris1986vectorized}. An effective strategy for embedding the NNG into an \textit{l}-dimensional grid is to embed the individual components of the graph separately. The connected components of NNGs capture clusters of samples.  Recursively building 1-NNGs on the previously obtained connected components provides a hierarchical view on how samples are merged together at successive levels. Considering each connected component as a node in the hierarchy one can identify complete paths on how these nodes and their associated samples successively merge together from bottom to top. Such a hierarchical node graph provides a view of data in terms of how the local neighborhood is distributed in the high dimensional space. After an inexpensive preliminary projection of the high dimensional data on a desired low dimensional space we can use the original hierarchical node graph in the target space to enforce the local structure directly. We achieve this with a fast recursive top-down approach by moving the clusters of samples towards these nodes starting at the top level with the least number of nodes and moving progressively downward to reach the bottom \ie the finer level. 

Since our proposal is rooted in obtaining the Hierarchical 1-Nearest Neighbor graph based Embedding, we term the method \textbf{h-NNE}. Figure~\ref{fig:teaser} depicts an example of embedding the ResNet-50 features of the full ImageNet dataset. As seen, in comparison to the current state-of-the-art methods, not only do we achieve competitive embedding quality - indicated by the trustworthiness metric - but also a significantly faster run-time. To summarize, our main contribution is an alternative dimensionality reduction and visualization method which does not rely on expensive optimization methods. This makes it operate at a magnitude faster than existing methods, without requiring hyperparameter tuning, and maintaining similar performance.
In the following sections, before delving into the proposed method, we will discuss the related works to place it in context and followed by experiments and comparisons with the state-of-the-art on a diverse collection of datasets. 
% %\clearpage
% \setlength{\belowcaptionskip}{-12pt}
\setlength{\abovecaptionskip}{-4pt}

\begin{figure*}
\begin{center}
\includegraphics[width=\linewidth]{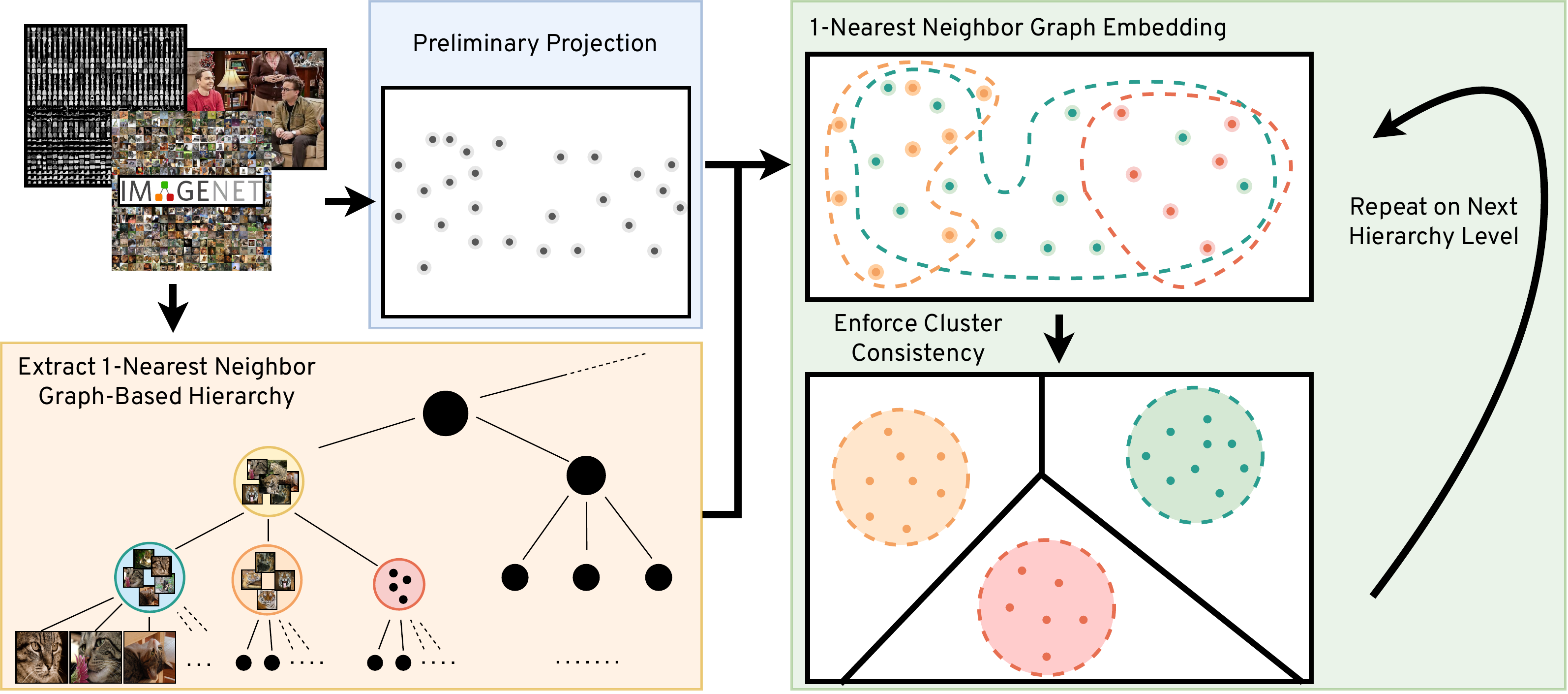}
\end{center}
\caption{
A summary of the h-NNE projection method.
}
\label{fig:overview}
\end{figure*} 

\section{Related work}
\label{sec:related}
Most of the current state-of-the-art unsupervised dimension reduction techniques aim at preserving the local pairwise distances in a projected manifold. 
Hinton and Roweis~\cite{hinton2002stochastic} introduced the Stochastic Neighborhood Embedding (SNE),
which creates a low-dimensional embedding by enforcing the conditional probabilities (euclidean distances converted to similarities) in the lower dimension to be similar to those in the higher dimension. This is achieved by fitting Gaussian distributions on the samples and matching them to distributions in the lower dimension. Building on SNE, t-Distributed Stochastic Neighbor Embedding (t-SNE)~\cite{tsne} substituted the Gaussian distribution, used in the low-dimensional space, with long-tailed t-distributions. t-SNE is inherently computationally expensive. Follow-up works of t-SNE such as Barnes-Hut t-SNE ~\cite{bh_tsne}, viSNE~\cite{amir2013visne}, FIt-SNE~\cite{linderman2019fast,linderman2017efficient}, and opt-SNE~\cite{belkina2019automated}, have further improved t-SNE so it can converge faster and scale better to larger datasets. Much of the followup work follows the same direction. For example, LargeVis~\cite{tang2016visualizing} and the current state-of-the art methods like  UMAP~\cite{umap} and PaCMAP~\cite{pacmap} have objective functions building on that of t-SNE. These methods focus on improving efficiency and preservation of more global structure along with the local structure. This is commonly achieved by construction and operations on weighted graphs. Typically a k-NN graph is constructed in the original space and then used to enforce the k-neighborhood of each sample point in a preliminary projected space. All current state-of-the-art methods such as t-SNE~\cite{tsne}, LargeVis~\cite{tang2016visualizing},  UMAP~\cite{umap} and PaCMAP~\cite{pacmap} are heavily related in terms of this underlying process. A more thorough discussion on this k-NN graph optimization view is provided in McInnes~\etal~\cite{umap} that compares t-SNE, LargeVIs and UMAP optimization equations in this context. The difference between these techniques comes from the changes in the objective function that is used for optimizing the projected embedding. For instance a common goal of optimization is to attract together samples that lie in the same k-neighborhood and repulse them from far away samples. If viewed in terms of the k-NN graph layout this amounts to defining and using a set of attractive forces applied along edges and a set of repulsive forces applied among vertices. This is a non-convex optimization problem and convergence is achieved by carefully changing the attraction and repulsion forces through gradient descent based learning. In effect it is contrastive learning by using samples in k-neighborhoods as positives to the current data point and sampling negatives from the rest. To reduce cost, UMAP employs negative sampling to pick a subset of samples for the repulsive force whereas the more recent PaCMAP defines highly-weighted near-pairs and  mid-near pairs to help preserve more structure. A related strategy is used by triplet constraint methods such as TriMap~\cite{amid2019trimap}
which is initialized with the low dimensional PCA embedding. This embedding is then modified using a set of carefully selected triplets from the high-dimensional representation. Finally, t-SNE and UMAP methods have their parametric versions~\cite{van2009learning, pumap} where they train an MLP with a similar objective function. The parametric methods simplify the projection of new datapoints. 
Another important consideration in these methods is their reliance on user supplied hyperparameters. Apart from specifying the number of neighbors k to construct the graph, several other optimization specific parameters are also required. 

Our approach is a major shift in contrast to these current methods. Instead of building a weighted k-NN graph, we create a clustering hierarchy by recursively building 1-NN graphs with static edge links.
This hierarchy is then used to move samples in the lower dimensional space without requiring the use of gradient descent based optimization. Simultaneously, this removes the reliance on specific hyperparameters. This provides us with a highly efficient and scalable dimension reduction method. 

\section{The h-NNE algorithm}

Our projection algorithm consists of three main steps: building a tree hierarchy based on 1-NNGs, computing a preliminary projection with an approximate version of PCA and adjusting the projected point locations based on the constructed tree. The projected point location adjustment can be enhanced with an optional inflation step which can be used to improve visualization. In the following sections, we will elaborate on each step and provide some evidence of their validity. Figure \ref{fig:overview} gives an overview of the method.

\subsection{Nearest neighbor graph based hierarchy}

Several projection methods start by defining a structure over the data which encodes the relative positions of points and then project in a way that preserves this structure. For example, UMAP relies on a weighted graph encoding nearest neighbor relations while t-SNE uses a collection of local distributions based again on the nearest neighbor relations. In our case, we strive for a structure which captures both local neighbor properties of points and global clustering properties. In order to achieve this with a low computational cost, and at the same time keep the approach simple and parameter free, we build a hierarchy based on 1-NN relations between points. This approach is inspired by classical work on nearest neighbor graphs such as~\cite{eppstein1997nearest} and the FINCH clustering method~\cite{finch}.

Assume that our dataset is $\mathbf{X} = \{\mathbf{x}_i\}_{i\leq N}$, where $\mathbf{x}_i\in\mathbb{R}^D$. The first step in constructing the hierarchy entails building $NNG(\mathbf{X})$ which is a directed graph that connects each point to its nearest neighbor. This can be performed by using any nearest neighbor or approximate nearest neighbor algorithm. Next, we identify the connected components of $NNG(\mathbf{X})$, denoted by $\{NNG_i(\mathbf{X})\}_{i\leq g_0}$, which form directed graphs with all edges pointing to a single bi-root. For each graph $NNG_i(\mathbf{X})$, we compute its centroid $\mathbf{c}_i^{(0)} = \frac{1}{g_0} \sum\limits_{\mathbf{x} \in NNG_i(\mathbf{X})} \mathbf{x}$ and thus form a new set of points $\mathbf{C}^{(0)} = \{\mathbf{c}^{(0)}_i\}_{i\leq g_0}$. We then repeat the same process of computing $NNG(\mathbf{C}^{(0)})$ and its components' centroids to derive $\mathbf{C}^{(1)}$ and continue until we reach the smallest set of centroids $\mathbf{C}^{(l)}$ which contains at least three points. The NNG hierarchy is then the tree $T_{NNG}(\mathbf{X}) = \langle \bigcup\limits_{i\leq l}\mathbf{C}^{(i)}\cup \mathbf{X}, E \rangle$, where each centroid is connected to each of the points of the NNG component which corresponds to it. Figure \ref{fig:nng_hierarchy} displays a single step of this iterative process.

In comparison with k-NNGs for $k > 1$, 1-NNGs are quite small in size, which make them well-suited to construct a fine-grained hierarchy with several levels. A simple rule of thumb is that the number of points decreases on average by a factor of 0.31 on every step. This can be extracted from the following theorem assuming that, at least locally, the data and centroids on higher levels are uniformly distributed. Though the validity of the assumption is hard to verify, we notice that this rule holds for all observed datasets.

\begin{theorem}[Eppstein, Paterson, Yao~\cite{eppstein1997nearest}]
The expected number of components in $NNG(\mathbf{X})$ for a uniform random point set $\mathbf{X}$ in
a unit square is asymptotic to approximately $0.31 |\mathbf{X}|$.
\end{theorem}

\subsection{Preliminary linear embedding}

Though the clustering tree structure is enough to produce a projection respecting the partitions of the original dataset on separate levels, the relative positions of the points contained in the graph components $NNG_i(\mathbf{X})$ and $NNG_i(\mathbf{C}^{(k)})$ need to be determined in the target space. One could use a random projection, or even start with random points, but there is an extra gain in both preservation scores and in visual quality when using some meaningful initial projection. We choose to use PCA and to accelerate its computation, we estimate the covariance matrix of the data using the centroids $\mathbf{C}^{(i)}$ of a predefined level of the hierarchy. One could as well sub-sample the original point set $\mathbf{X}$, but we notice that using the centroids increases stability and avoids deviation of the initialization between runs. We experimentally verify that this approximation of principal components produces results comparable to using PCA on the full data. We provide this analysis in the supplementary.

We choose the level of centroids to be the lowest level of the hierarchy such that all levels above have cardinality less than 1000. This implies that if the dataset is small, thus for all $i$, $|\mathbf{C}^{(i)}| < 1000$, then the PCA will directly be computed on $\mathbf{X}$. The advantage of this approximation is that we can reduce the computational cost of PCA from $\mathcal{O}(N\cdot D^2)$ to $\mathcal{O}(D^2)$ with $N$ replaced with a factor of ~1000. 

Once the eigenvectors of PCA are computed, say in a matrix $V$, then all points of $\mathbf{X}$ and all centroids $\mathbf{C}_i$ are projected from the higher dimension $D$ to the lower dimension $d$ by multiplying with this single shared matrix. We denote the projections of such points by a tilde superscript, for example $\tilde{\mathbf{c}}_i^{(k)}$ for centroids and $\tilde{\mathbf{x}}_i$ for points of the dataset.

\setlength{\belowcaptionskip}{6pt}
\setlength{\abovecaptionskip}{0pt}

\begin{figure}[t]
\begin{center}
% \begin{overpic} 
% [width=\linewidth]
% {example-image-a}
% \end{overpic}
\includegraphics[trim={0, 2.33cm, 0, 0}, clip,width=\linewidth]{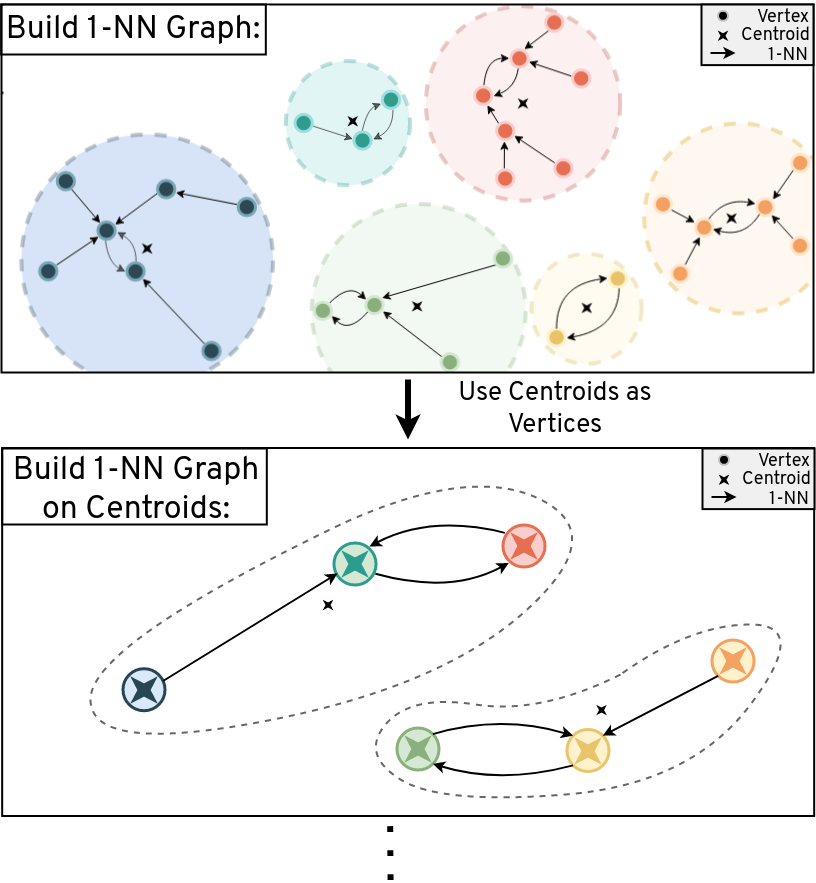}
\end{center}
\caption{
The induction step in building $T_{NNG}(\mathbf{X})$. All 1-NNG components are mapped to their centroids which form the basis for the next step.
}

\label{fig:nng_hierarchy}
\end{figure}

\subsection{Hierarchical point translation}

This is the central part of the algorithm and the goal is to move the points hierarchically so that they occupy the projection space $\mathbb{R}^d$ following the tree $T_{NNG}(\mathbf{X})$ in a way that the 1-NN relationships are preserved over all levels. Once we have a preliminary projection for all points and centroids, we start from $\mathbf{C}^{(l)}$ and consider its projected centroids $\{\tilde{\mathbf{c}}_i^{(l)}\}_{i\leq g_l}$. Those centroids form a Voronoi tessellation of $\mathbb{R}^d$ and an ideal way to place the lower level projected centroids $\{\tilde{\mathbf{c}}_i^{(l-1)}\}_{i\leq g_{l-1}}$ would be to select for each $\tilde{\mathbf{c}}_i^{(l)}$ the centroids of level $l-1$ which correspond to it, translate them so that they are centered around $\tilde{\mathbf{c}}_i^{(l)}$ and finally spread them to occupy the corresponding Voronoi cell. In order to perform this process in an efficient and easy to vectorize way, we choose to use the already known distance $d_i^{(l)}$ of each $\tilde{\mathbf{c}}_i^{(l)}$ to its nearest neighbor in $\tilde{\mathbf{C}}^{(l)}$ and then scale the translated, lower level centroids to a d-ball of radius $\frac{1}{3} d_i^{(l)}$. This distance guarantees that points belonging to neighboring centroids will not form nearest neighbor relationships cross-centroids, thus preserving the separation encoded in $T_{NNG}(\mathbf{X})$. Figure \ref{fig:voronoi_with_radius} illustrates this process.

Once the points of $\tilde{\mathbf{C}}^{(l - 1)}$ are placed around the points of $\tilde{\mathbf{C}}^{(l)}$, we use them to translate the points of $\tilde{\mathbf{C}}^{(l - 2)}$ around them, the same way as before. This step is repeated until we reach the level of $\mathbf{X}$ which forms the final projection. 

There is still one issue we need to address. This is the fact that though the radius $\frac{1}{3} d_i^{(l)}$ guarantees the separation of neighboring centroids on one step, it could be the case that the borders of this d-ball are crossed by points moved on later step of the iteration (see again figure \ref{fig:voronoi_with_radius}). Below we compute a shrinking coefficient for this radius, such that this guarantee still holds for the points moved in later steps.

\begin{lemma}\label{radius-bound}
Given a d-ball $B(\mathbf{c}^{(k)}_i, r)$ centered on a centroid, all points belonging to $\mathbf{c}^{(k)}_i$ translated with the h-NNE algorithm with radii multiplied a factor of $\frac{3}{5}$, lie inside $B(\mathbf{c}^{(k)}_i, r)$.
\end{lemma}

\begin{proof}
Assume that a factor $s$ is used to reduce the computed radii on each step of h-NNE. On the first step, all lower level centroids are translated and scaled so that they lie in a d-ball of radius $sr$. The worst case scenario for the next step is that there a two antidiametrical points $\mathbf{c}^{(k)}_1$, $\mathbf{c}^{(k)}_2$ which are nearest neighbors. In that case, the points of the next step belonging to $\mathbf{c}^{(k)}_1$ will be placed inside a d-ball around it of radius $\frac{2}{3} sr$ since $d(\mathbf{c}^{(k)}_1, \mathbf{c}^{(k)}_2) = 2sr$. This means that the largest distance of any of those points to $\mathbf{c}^{(k)}_i$ is $sr + \frac{2}{3} s^2r$. By recursively computing those worst case scenarios, we get that for infinite steps of the algorithm, the most distant point will be placed in a distance of at most 

\begin{equation}
\sum\limits_{j\in \mathbf{N}} (\frac{2}{3})^j s^{j+1}r = sr \sum\limits_{j\in \mathbf{N}} (\frac{2}{3}s)^j = \frac{sr}{1 - \frac{2}{3}s}
\end{equation}

Therefore, in order for all points to lie inside the original ball $B(\mathbf{c}^{(k)}_i, r)$, we need that $\frac{sr}{1 - \frac{2}{3}s} \leq r$ from which we get that $s \leq\frac{3}{5}$.
\end{proof}

The above bound guarantees that if we place new points in a ball of radius $\frac{3}{5}\cdot\frac{1}{3} = 0.2$ times the distance to the nearest neighbor, then the nearest neighbors of points will be restricted in the clusters formed in the NNG hierarchy. In practice the worst case scenario of anti-diametrical points does not occur so often. In low dimensions where points are more dense one can use radii of $\frac{1}{3}$ of the 1-NN distance or even more without noticeable drop in k-NN preservation. This can be particularly useful for visualization, as it can help make plots more spread on the plane or 3D space.

\setlength{\abovecaptionskip}{8pt}

\begin{figure}[t]
\centering
\includegraphics[width=\linewidth, trim={0, 0, 0, 1mm}, clip]{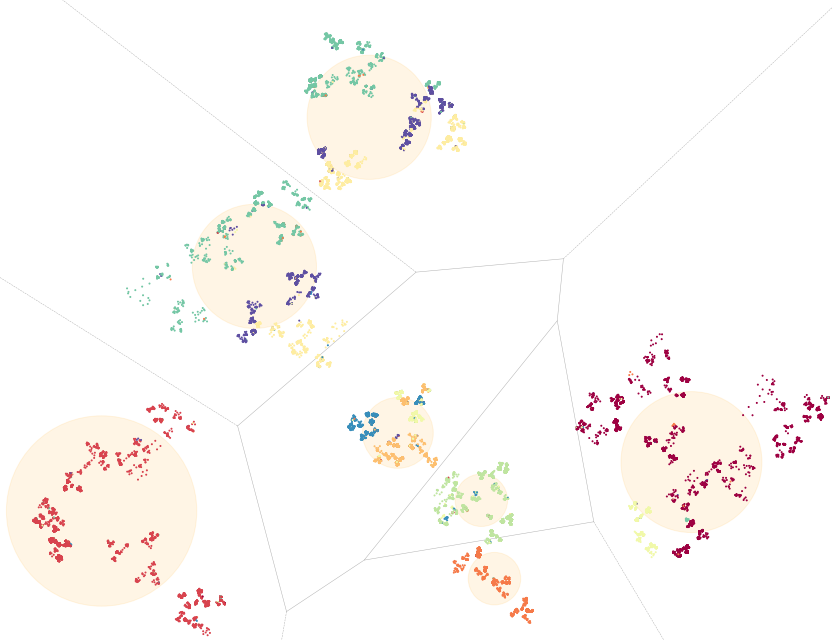}
\caption{
Voronoi cells of the top level centroids in MNIST. The circles have radius $\frac{1}{3}$ the distance from the nearest neighbor and the points are projected with a shrinking factor of exactly this factor of $\frac{1}{3}$. One can notice that the final points cross the boundaries of the circle. Nevertheless, the density of the data in 2D result to a situation where no severe overlaps appear.
}
\label{fig:voronoi_with_radius}
\end{figure}

\textbf{Point cluster inflation for visualization purposes.} The use of a single linear projection for all points can result in cluttered point clusters when they are not well aligned to the global principal components used to project. Though this has minimal impact to performance, it leads to poorer visualization which contains artifacts from this initial projection. In order to enhance the shape of images without sacrificing speed and without adding new hyper-parameters, we add the option to inflate potentially squeezed point clusters using six local rotations with angles equally distanced in the interval $[0, \frac{\pi}{2}]$, followed by a scaling and the inverse rotation. This results in an output almost equivalent to that of rotating the clouds to their PCA principal components, scaling them, and then rotating them back to the original orientation but much less computationally expensive.

\setlength{\belowcaptionskip}{-6pt}
\setlength{\abovecaptionskip}{-1pt}

\begin{table*}[t]
%\small
%\tabcolsep=0.00001cm
\begin{center}
 %\vspace{-0.5cm}
\resizebox{17cm}{!} {
\begin{tabular}{l |c|c|cccc|cc|cc}
\toprule
  					& Higgs~\cite{baldi2014searching} & Google News~\cite{mikolov2013distributed} & COIL 20~\cite{nene1996columbia} & CIFAR-10~\cite{cifar} & F-MNIST~\cite{f_mnist} & ImageNet~\cite{imagenet} & BBT~\cite{bbt} & Buffy~\cite{bbt}   &  \multicolumn{2}{c}{MNIST~\cite{mnist, mnist8m}} \\
\midrule
Type  				& Sensor  &	Text& \multicolumn{4}{c|}{Objects} &	\multicolumn{2}{c|}{Videos} &\multicolumn{2}{c}{Digits}  \\
\#Classes  			    & 2&	-&	20&	10&	10 & 1000 & 5 & 6 & \multicolumn{2}{c}{10}  \\
\#Samples  			& 11M&	3M& 1440& 60K& 70K& 1.2M&	200K & 206K & 	70k & 8M  \\
Dimension  			& 28 &	300 & 16384 & 3072 & 784& 2048& 2048&	2048&   784 & 784   \\

Features  & Measurements & Word2Vec &Pixels &Pixels & Pixels & ResNet50 & ResNet50 &ResNet50 & Pixels & Pixels \\
\bottomrule
\end{tabular}
}
\end{center}
%\vspace{-0.25cm}
\caption{\small{Datasets of different size ranging from 1400 to 11 million samples in 28 to 16384 dimensions used in experiments.}}\label{tab:datasets}
%\vspace{-0.15cm}
\end{table*}
\setlength{\abovecaptionskip}{6pt}

\begin{table}[b]
\centering
\resizebox{\linewidth}{!}{ %< auto-adjusts font size to fill line

\begin{tabular}{@{}lccccc@{}}
\toprule
 & h-NNE & t-SNE & fIt-SNE & UMAP & PaCMAP \\
\midrule
COIL20 & 0.994 & 0.998 & \textbf{0.998} & 0.996 & 0.985 \\
MNIST & 0.983 & 0.989 & \textbf{0.991} & 0.958 & 0.950 \\
F-MNIST & 0.981 & 0.991 & \textbf{0.992} & 0.977 & 0.966 \\
CIFAR10 & 0.907 & 0.927 & \textbf{0.932} & 0.829 & 0.818 \\
BBT & 0.982 & 0.99 & \textbf{0.99} & 0.966 & 0.958 \\
Buffy & 0.976 & 0.988 & \textbf{0.99} & 0.954 & 0.952 \\
ImageNet & 0.928 & $-$ & \textbf{0.956} & 0.895 & 0.822 \\
HIGGS & 0.849 & $-$ & \textbf{0.979} & 0.909 & 0.899 \\
MNIST 8M & \textbf{0.967} & $-$ & 0.956 & $-$ & 0.945 \\

\bottomrule
\end{tabular}

} % \resizebox
\caption{
Local structure preservation : Trustworthiness.
} % \caption
\label{tab:sota}
\end{table}

\textbf{Projecting new points.} 
The projection of new points can be performed by repeating the same algorithm as before, just for the individual points and by descending only the second level of the hierarchy. If $\mathbf{x}$ is the new point, we first identify the closest centroid in $\mathbf{C}^{(1)}$ using our ANN algorithm of choice. Then the point is transformed by scaling, applying the pre-computed PCA and normalizing the position of the point relative to the centroid based on the corresponding radius. If point cluster inflation was used in the original projection, then the relevant rotations and scalings are also performed before the final normalization step.

\subsection{Computational complexity}
\label{subsec:method_complexity}
There are three steps in h-NNE which make up its complexity: the construction of the $T_{NNG}(\mathbf{X})$ tree, the preliminary PCA projection and the hierarchical point translation. Since each component of the NNG graph contains at least two points, the height of $T_{NNG}(\mathbf{X})$ is $\mathcal{O}(\log N)$. That implies that given that the computation of approximate 1-NN is at least $\mathcal{O}(N)$, no matter which method is used, the total complexity is equal to the complexity of a single Approximate Nearest Neighbor (ANN) step including any preparation on the dataset (e.g. building indexing trees) and querying once on all points of the dataset. We denote the ANN complexity with $\mathcal{O}(ANN(N, D))$. The PCA projection step requires $\mathcal{O}(Dd^2)$, given we have fixed the number of used samples to a constant number of points. Finally, the point translation steps are $\mathcal{O}(N \log N d + ANN(N, d))$, again because of the logarithmic height of the tree, the we use group by operations, and the fact we need to compute nearest neighbors to find the radius of the d-balls where clusters are expanded. Thus, overall, the algorithm has $\mathcal{O}(ANN(N, D) + Dd^2 + N \log N d)$ complexity. 

The complexity $ANN(N)$ is not easy to compute. In the worst case scenario, it could be $N^2$ when using linear search. Some of the older exact methods~\cite{Omohundro89fiveballtree, kdtreepaper} achieve $N \log N$ complexity, but unfortunately scale exponentially on $D$. Approximate methods such as HNSW~\cite{Malkov2020EfficientAR}, NNDescent~\cite{nndescentpaper} and ScaNN~\cite{pmlr-v119-guo20h} achieve good performance on real-world datasets but support it with experimental evidence and no complexity bounds. In our current implementation we use PyNNDescent~\cite{pynndescent-repo} which is a tuned version of NNDescent. The empirical complexity of NNDescent is approximately $O(N^{1.14})$ for datasets with small intrinsic dimensionality. Finally, an interesting analysis~\cite{Baron2019KNearestNA} provides synthetic families of datasets where NNDescent attains a complexity of $O(N^2)$ and $O(N \log N)$ respectively. Further work of the authors~\cite{empirical-complexity-nndescent-baron}, suggests an empirical complexity of $K^2 N\log N$, where $K$ is the number of neighbors.

\section{Experiments}

We demonstrate h-NNE on diverse datasets of different size which cover domains such as sensor data, text, digits, videos and objects. We first introduce the datasets and performance metrics, followed by a thorough comparison of the proposed method to current state-of-the-art algorithms.
\subsection{Evaluation}
\textbf{Datasets.} 
The datasets are summarized in Table~\ref{tab:datasets}. Apart from some commonly used datasets we also include few to test against size and dimensions.  Altogether, we use 9 datasets ranging from 1440 to 11 million samples in 28 to 16384 dimensions. We provide more details for each dataset in the supplementary.

\textbf{Metrics.} 
A well studied issue in dimensionality reduction is how to balance local and global structure preservation~\cite{de2002global,kobak2021initialization}. Methods like t-SNE are well known for their local structure preservation. More recent methods such as UMAP~\cite{umap, pumap} and PaCMAP~\cite{pacmap} strive to preserve both local and global structure. 
We therefore evaluate all methods on the considered datasets in these two aspects.

\begin{table}[b]
\centering
\resizebox{\linewidth}{!}{ %< auto-adjusts font size to fill line

\begin{tabular}{@{}lccccc@{}}
\toprule
 & h-NNE & t-SNE & fIt-SNE & UMAP & PaCMAP \\
\midrule
COIL20 & \textbf{0.799} & 0.778 & 0.769 & 0.705 & 0.758 \\
MNIST & 0.671 & 0.711 & 0.748 & \textbf{0.804} & 0.713 \\
F-MNIST & \textbf{0.925} & 0.784 & 0.848 & 0.848 & 0.866 \\
CIFAR10 & 0.911 & 0.921 & 0.929 & 0.927 & \textbf{0.932} \\
BBT & 0.644 & 0.600 & \textbf{0.711} & 0.556 & 0.666 \\
Buffy & \textbf{0.857} & 0.571 & 0.543 & 0.657 & 0.704 \\
ImageNet & 0.604 & $-$ & 0.639 & 0.605 & \textbf{0.646} \\
MNIST 8M & \textbf{0.774} & $-$ & 0.744 & $-$ & 0.735 \\

\bottomrule
\end{tabular}      
 
} % \resizebox
\caption{
Global structure preservation: Centroid Triplet Accuracy.
} % \caption
\label{tab:cta}
\end{table}

\noindent\textbf{Local structure preservation:} commonly measured by leave-one-out cross validation using the \textit{k-NN classifier}. \textbf{k-NN accuracy} is a standard evaluation method for dimension reduction methods as this measures if the classification accuracy based on neighborhoods would remain close to that in the original space.
As in previous studies, we use a 10-fold stratified cross-validation to measure k-NN accuracy on varying number of k values. Another closely related local structure preservation metric is Trustworthiness. \textbf{Trustworthiness}~\cite{venna2006local} penalizes for each point every one of its k nearest neighbours in the embedding space by the amount its rank in the original space exceeds k. Trustworthiness is defined as 
\begin{equation*}
T(k) = 1 - \frac{2}{nk (2n - 3k - 1)} \sum^n_{i=1}\sum_{j \in\mathcal{N}_{i}^{k}} \max(0, (r(i, j) - k))
\end{equation*}    
 and is scaled between 0 and 1, with 1 being more trustworthy. We use scikit-learn's implementation with $k=5$.

\setlength{\abovecaptionskip}{0pt}
\begin{figure*}[t]
\begin{center}
% \begin{overpic} 
% [width=\linewidth]
% {example-image-a}
% \end{overpic}
\includegraphics[width=\linewidth]{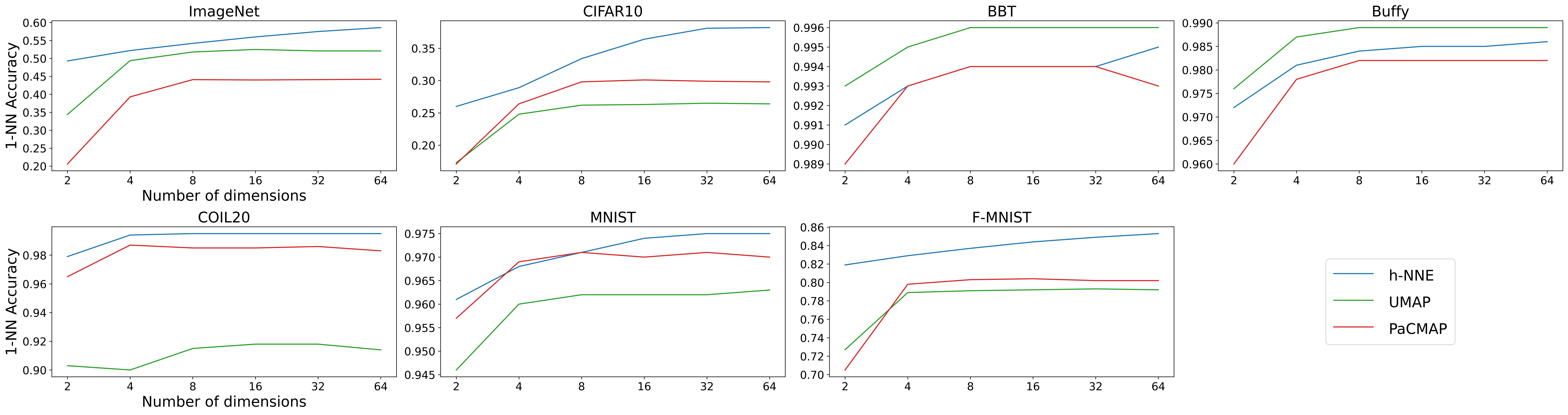}
\end{center}
\caption{
Projection on different dimensions: KNN accuracy for 2, 4, 8, 32 and 64 dimensions}
\label{fig:dim_ablation}
\end{figure*}

\noindent\textbf{Global structure preservation:} To measure the preservation of global structure i.e. relative positioning of individual neighborhoods Wang~\etal~\cite{pacmap} proposed a metric that measures how well the distribution of distances between class centers in the original space are preserved in the embedding space. It is obtained by  computing the centroids and forming all possible triplets between them. The metric \textbf{Triplet Centroid Accuracy} measures the percentage of triplets whose relative distance in the high- and low-dimensional spaces maintain their relative order.  

\subsection{Comparison with state-of-the-art}
We compare with representative current best methods for unsupervised dimensionality reduction and visualization. t-SNE~\cite{tsne}, owing to its remarkable local structure preserving properties and long standing follow up work, is a default choice for 2D visualizations. It has, however, high compute overhead. We use its optimized version Barnes-Hut t-SNE~\cite{bh_tsne} in our comparison. Since this still does not easily scale to data beyond 200K in size, we also include a much faster and scalable variant FFt-accelerated Interpolation-based t-SNE (FIt-SNE)~\cite{fitSNE}. FIt-SNE is the currently best t-SNE method in terms of speed and performance. Of note,  t-SNE is limited in that its running time increases exponentially to the number of dimensions and though FIt-SNE could in principal support higher dimensions it's official implementations do not support it.  Among other recent methods, we include the popular UMAP~\cite{umap} and a very recent method, PaCMAP~\cite{pacmap}, which builds upon the strengths of t-SNE and UMAP.

Table~\ref{tab:sota} provides the trustworthiness scores of all methods on the datasets. t-SNE and FIt-SNE both are quite good with respect to preserving the local structure. h-NNE follows closely to t-SNE methods in terms of the local structure preservation. UMAP and PaCMAP perform on par though slightly worse than h-NNE.
We provide the k-NN classifier accuracy and comparisons in the supplementary. 

On the global structure preservation metric as shown in Table~\ref{tab:cta} the centroid triplet accuracy for h-NNE embeddings is on par as well across all datasets. Since these metrics require ground truth labels of the datasets the results on Google News dataset cannot be computed. Similarly, the centroid triplet metric on HIGGS is not evaluated as it has only two classes.  These results indicate that, overall, h-NNE is highly competitive with these current methods in terms of both local and global structure preservation.
\setlength{\belowcaptionskip}{-0pt}

\begin{table*}[t]
\small
\centering
\resizebox{15cm}{!}{ %< auto-adjusts font size to fill line
\begin{tabular}{l | l c| cccccc}
\toprule
 Dataset & \#S & Dim & h-NNE (ours) & t-SNE & UMAP & PaCMAP & FIt-SNE \\
\midrule

COIL20 &1440 &16384& \textbf{00:01} & 00:08 & 00:18 & 00:02 & 00:07 \\
MNIST & 70K & 784& \textbf{00:09} & 05:28 & 00:54 & 00:37 & 01:24 \\
F-MNIST &70K &784 & \textbf{00:07} & 05:37 & 00:58 & 00:38 & 01:20 \\
CIFAR10 &60K &3072& \textbf{00:14} & 08:35 & 01:02 & 00:39 & 02:24 \\
BBT &200K & 2048 & \textbf{00:58} & 57:32 & 04:08 & 01:38 & 04:50 \\
Buffy &206K & 2048 & \textbf{00:50} & \color{red}{01:01:00} & 04:18 & 01:42 & 05:07 \\
Google News &3M & 300 & \textbf{03:34} & $-$ & \color{red}{01:19:05} & \color{red}{01:14:04} & 52:16 \\
ImageNet &1.2M & 2048 & \textbf{05:11} & $-$ & 49:10 & 30:20 & 39:29 \\
HIGGS &11M & 28 & \textbf{12:25} & $-$ & \color{red}{03:53:04} & \color{red}{05:07:34} & \color{red}{02:49:28} \\
MNIST\_8M & 8M & 784 & \textbf{21:52} & $-$ & $-$ & \color{red}{03:08:47} & \color{red}{02:56:14} \\
\midrule
Framework  		& 		&   	&Python	&Python	&Python	&Python &C++ \\
\bottomrule
\end{tabular}       
} % \resizebox

\vspace{0.3cm}
\caption{
\small{Run-time comparison of h-NNE: We report the run time in {\color{red}{HH:MM:SS}} and MM:SS. $-$ denotes out of memory. }
} % \caption
\label{tab:time_comparison}
%\vspace{-0.5cm}
\end{table*}

\begin{figure*}
\begin{center}
% \begin{overpic} 
% [width=\linewidth]
% {example-image-a}
% \end{overpic}
\includegraphics[width=\linewidth]{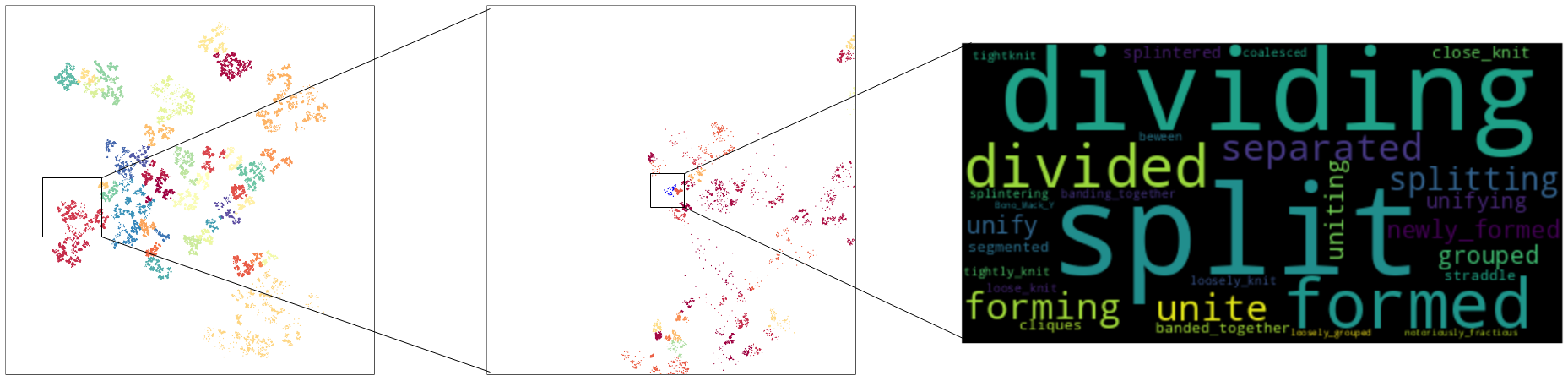}
\end{center}
\caption{
The h-NNE visualization of the unlabelled Google News embeddings. On the left, we see the cluster labels on the top level of the hierarchy. In the middle we zoom into one location and isolate a small cluster using cluster labels from an earlier level of the hierarchy. 
}
\label{fig:google-news-zoom}
\end{figure*}

\textbf{Projections of varying dimensions.}
There is no restriction on the target dimension for h-NNE. Figure~\ref{fig:dim_ablation} demonstrates projection on different dimensions. For this experiment we include 7 datasets and compare with UMAP and PaCMAP as only these support higher dimensions. Here k-NN ($k=1$) accuracy across different datasets shows that h-NNE has a consistent upward trend in performance when moving to higher dimensions.

\textbf{Computational performance comparisons.}
Benchmarks on all datasets were performed on a workstation with an AMD Ryzen Threadripper 2990X 32-core processor with 128 GB RAM. Table~\ref{tab:time_comparison} provides the total run-time of each algorithm in comparison. As expected t-SNE was not able to scale on large data size and UMAP ran out of memory on 8 million 784 dimensional data. PaCMAP and FIt-SNE - the fast alternative approach to t-SNE which is a highly optimized and parallelized C++ implementation - both ran on all datasets. h-NNE has a clear advantage in terms of speed and it scales really well on the size of the dataset, reaching speedups well above ten-fold on the larger datasets. The time efficiency of our method comes from the approach itself (i.e. not relying on expensive gradient decent-based embedding optimizations) and could be sped-up further with similar parallelized implementations. On the HIGGS data with 11M samples h-NNE takes $\sim$ $13$ minutes vs almost $3$ hours for FIt-SNE and $4 - 5$ hours for UMAP and PaCMAP. Since all these existing methods need to build weighted k-NN graphs with the requirement to store and access pairwise distance floats for their optimizations, their memory requirement is very high. For instance as opposed to h-NNE none of these methods could run on a machine with 64 GB ram at million scale. For more details on the computational complexity we refer to section~\ref{subsec:method_complexity}.

\textbf{Exploring unlabeled data.}
We see an advantage of h-NNE in exploring unlabeled data because of its hierarchical nature. On the one hand, different clusters of the data are visually separated, and on the other hand, we have cluster-labels generated on all levels of the hierarchy. In figure~\ref{fig:google-news-zoom} we color the Google News dataset projection with those cluster-labels on two different levels and demonstrate a zooming process where we isolate a cluster of similar word vectors.

\textbf{Limitations.} %
One limitation of our algorithm is its dependence on the hierarchical NNG structure we build. This structure clusters the data on each level and the final projection is based on those clusterings. This means that any error that might occur in those clusterings is directly translated to reduced quality of the global structure and localization of the points. On the bright side, those classification errors can reflect poor separation of classes in the original space, which is a desired property when one is using h-NNE to inspect the quality of features generated for datasets.

A second limitation is that our method does not preserve by design the topology of the original space. The fact that we use 1-NNs to build our hierarchy implies that any shape homeomorphic to say $S^1$ will only be preserved if it is contained in a single 1-NN graph component on the lowest level of the tree and if the PCA projection is preserving its circular form. As the size of 1-NNG components is usually quite small this will seldom be the case. In order to remove this limitation one would need to preserve more local properties of the data. This would be possible on higher dimensions, but very low dimensions, like 2 or 3, it conflicts with the preservation of the hierarchical NNG data partitioning which is capturing a more global structure. Therefore, in this case we see this limitation more as a trade-off.

\section{Conclusions}
\vspace{-0.1cm}
We have presented an efficient dimensionality reduction algorithm which utilizes nearest neighbor graph and its hierarchical groupings of points. Using the hierarchy tree nodes as anchor positions in a preliminary projected space, we device a fast mechanism to move the points directly in the embedding space such that they preserve their local neighborhoods. The embedding process is optimization-free and does not rely on specifying hyperparameters. This results in a demonstrably fast and scalable technique which is shown to preserve both the global and local structure of the data well. The major benefit that our method offers, in addition to its speed and quality, is its ability to expose the clustering structure of the data both in the original and the target space. This may enable analyzing data at different levels of its hierarchical groupings. Faster run-time and the ability to provide clustering labels could be particularly useful for visualizing the large-scale unlabelled data. 

We expect that our work will enable the analysis of large scale data under reasonable time consumption and hope to trigger interest in the hierarchy-based study of global structural properties of the datasets.

%%%%%%%%% REFERENCES
% {
%     % \clearpage
%     \small
%     \bibliographystyle{ieee_fullname}
%     \bibliography{macros,main}
% }

% --- supplementary material

%\input{tab/datasets}

%%%%%%%%% REFERENCES
 {
    %   \clearpage
     \small
     \bibliographystyle{ieee_fullname}
     \bibliography{macros,main}
 }
% --- uncomment 
\appendix

% --- PDF will be split by an editor (e.g. macOS preview), so need to restart from page 1
\setcounter{page}{1}

% --- repeat the title (AT: haven't found a more elegant way to do this...)
\twocolumn[
\centering
\Large
\textbf{Hierarchical Nearest Neighbor Graph Embedding for Efficient Dimensionality Reduction} \\
\vspace{0.5em}Supplementary Material \\
\vspace{1.0em}
] %< twocolumn
\appendix

\section{Datasets description}
We used $9$ datasets ranging from $1440$ to $11$ million samples in $28$ to $16384$ dimensions. Table 1 in the main paper provides a good overview. Here we include more details for each dataset:

\texttt{Higgs}~\cite{baldi2014searching}: Higgs bosons montecarlo simulations of kinematic properties measured by the particle detectors in the accelerator. The 
\texttt{Higgs} dataset has 11 million samples in 28 dimensions.

\texttt{Google News}~\cite{mikolov2013distributed}:  is a dataset of 3 million words and phrases
derived from a sample of Google News documents and embedded into a 300
dimensional space via word2vec. It is an unlabelled dataset and therefore we can not compute metrics on it. 

\texttt{COIL 20}~\cite{nene1996columbia}: is a set of 1440 greyscale images consisting of 20 objects under 72 different rotations spanning 360 degrees. Each image has size 128x128 pixels and is treated as a single 16384 dimensional vector for the purposes of computing distance between images.

\texttt{CIFAR-10}~\cite{cifar}: $32\times32$ pixels RGB images of 10 object classes. We treat each image as a $32\times32\times3 = 3072$ dimensional pixel vector.

\texttt{Fashion MNIST}~\cite{f_mnist}: or F-MNIST is a dataset of 28x28 pixel grayscale images of fashion items (clothing, footwear and bags). There are 10 classes and 70000 images in total. Each image is treated as a ($28\times28 = 784$) dimensional pixel vector.

\texttt{ImageNet}~\cite{imagenet}: The ILSVRC2012 ImageNet dataset. $1.2$ million images belonging to $1000$ classes. Each image is represented by a $2048$ dimensional feature vector extracted using a trained ResNet-50 network. 

\texttt{BBT}~(season 1, episodes 1 to 6) and \texttt{Buffy}~(season 5, episodes 1 to 6) are challenging video face identification/clustering datasets. They are generated based on the videos of the sitcoms 
\textit{The Big Bang Theory} and 
\textit{Buffy the Vampire Slayer} on small cast lists, for BBT: 5 main casts, and for Buffy: 6 main casts. The data comprises of detected faces in video frames represented by a trained (on the VGG-Face dataset) ResNet-50 model. BBT has a total of $199346$ frames and Buffy has $206254$ frames. The extracted ResNet-50 feature vectors are $2048$ dimensional. The data for BBT and Buffy are provided by~\cite{bbt}. Since the feature vectors are obtained from a trained CNN model on face dataset, one should expect to see $5$ main clusters (for BBT) and $6$ clusters (for Buffy) in the embedding space.

\texttt{MNIST \& MNIST-8M}~\cite{mnist, mnist8m}: is a dataset of 28x28 pixel grayscale images of handwritten digits. There are 10 digit classes (0 through 9).  We use two variants of MNIST. \texttt{MNIST} $70000$ (train + test) images  and \texttt{MNIST-8M}~\cite{mnist8m} $8.1$ million total images obtained by applying random transformations to each MNIST image. Each image is treated as a pixel concatenated $784$ dimensional vector.

\section{Projecting new points}

As mentioned at the end of section 3.3 of the main paper, we can easily project new points by following the original algorithm. Here we perform a comparison with  two of the other methods which also provide the option to project new points, namely FIt-SNE and UMAP. To create a realistic scenario, we create an embedding of the training part of the ImageNet dataset and based on the structure learned on it we project the test part of ImageNet. In Table~\ref{tab:new-points-performance}, we can see the performance of different projections of ImageNet validation set and the corresponding time required to project the 50000 vectors of dimension 2048.

\begin{table}[t]
\centering
\resizebox{\linewidth}{!}{ %< auto-adjusts font size to fill line

\begin{tabular}{@{}lccccccccccc@{}}
\toprule
& 1-NN ACC & Trustworthiness & CTA & Runtime\\
\midrule
UMAP & 0.238 & 0.752 & 0.595 & 1min 47s \\
FIt-SNE & 0.526 & 0.933 & 0.637 & 1min 8s \\
h-NNE (ours) &	0.518 & 0.937 & 0.618 &  19s \\

\bottomrule
\end{tabular}      
 
} % \resizebox
\caption{
New points projection: Performance and time comparison of new points projection using ImageNet validation set (50K samples) as new points and h-NNE, UMAP and FIt-SNE built on Imagenet Train set.
} % \caption
\label{tab:new-points-performance}
\end{table}

As we can see, both the performance and speed of our method is preserved on the new points projection.

\section{Preliminary projection: random initialization versus PCA}
As described in section 3.1 of the main paper, We initialize with a preliminary projection using PCA. To reduce the computational complexity of PCA we proposed to use PCA on a reduced number of samples by using the centroids obtained on a predefined level of the h-NN graph. Here we provide further analysis and comparison on this initialization. We include an ablation using 6 medium scale datasets (up till 1 million samples). We show the impact on performance using 4 initialization methods: 

\texttt{1. Random init}: We start with random uniformly distributed d-dimensional points.

\texttt{2. Random Projection init}: We project the original data in $\mathbb{R}^D$ to $\mathbb{R}^d$ with d random uniform vectors of D-dimension each.  

 For random projections initialization we compute results over 5 runs and report the average.
 %We do not show their standard deviations in Table~\ref{tab:init} as those were very small within $\pm 0.03$

\texttt{3. Full PCA init}: We use PCA on the full data to obtain the preliminary projection.

\texttt{4. PCA on centroids init}: proposed initialization used in the paper: We use PCA on the $\sim 1000$ points/centroids of the full data from our built 1-NN hierarchy graph to obtain a faster preliminary projection.

In Table~\ref{tab:init} we show a comparison on both local (Trustworthiness and KNN) and global (Centroid Triplet Accuracy (CTA)) structure preservation metrics. As seen the random initialization provides similar local structure preservation as the more time consuming PCA projection. However the random projections can not recover the global structure well in comparison (lower CTA scores). This is also depicted in a visual comparison of projections in Figure~\ref{fig:rand_init}. Figure~\ref{fig:rand_init} shows our projections on the \texttt{BBT} dataset which has 5 classes. Both PCA and faster PCA on Centroids has very similar outputs whereas the projection based on random inits. shows noisier global structure (splitting the same class).

\begin{table}
\centering
\resizebox{\linewidth}{!}{ %< auto-adjusts font size to fill line

\begin{tabular}{@{}lccc@{}}
\toprule
 & 1-NN ACC & Trustworth. & CTA \\
\midrule
\textbf{COIL20} \\
\midrule
Random & 0.988 & 0.988 & 0.577  \\
Random Proj &0.991 & 0.992 & 0.666 \\
PCA-full & 0.989 & 0.993 & 0.799  \\
PCA-centroids & 0.990 &  0.994 & 0.799  \\

\midrule
\textbf{MNIST} \\
\midrule
Random & 0.946 & 0.970 & 0.630 \\
Random Proj & 0.960 & 0.982 & 0.715  \\
PCA-full & 0.962 & 0.984 & 0.752  \\
PCA-centroids & 0.965 &0.983 & 0.671  \\

\midrule
\textbf{Fashion MNIST} \\ 
\midrule
Random & 0.782 & 0.926 & 0.564  \\
Random Proj &0.820 & 0.955 & 0.688  \\
PCA-full & 0.823 & 0.976 & 0.896  \\
PCA-centroids & 0.826 & 0.981 & 0.925  \\

\midrule
\textbf{BBT} \\
\midrule
Random & 0.985 & 0.946 & 0.529  \\
Random Proj &0.990 & 0.971 & 0.618   \\
PCA-full & 0.992 & 0.974 & 0.703   \\
PCA-centroids & 0.992 & 0.982 & 0.644   \\

\midrule
\textbf{Buffy} \\
\midrule
Random & 0.967 & 0.951 & 0.621  \\
Random Proj & 0.976 & 0.968 & 0.516   \\
PCA-full &0.982 & 0.975 & 0.867  \\
PCA-centroids & 0.975 & 0.976 & 0.857  \\

\midrule
\textbf{ImageNet} \\
\midrule
Random & 0.436 & 0.857 & 0.589  \\
Random Proj & 0.560 & 0.931 & 0.624   \\
PCA-full & 0.567 & 0.933 & 0.654   \\
PCA-centroids & 0.557 & 0.928 & 0.604   \\
\bottomrule
\end{tabular} 
   
} % \resizebox
\caption{
Ablation: Impact of preliminary projection.
} % \caption
\label{tab:init}
\end{table}

\section{Impact of point cluster inflation for visualization}
In the main paper near the end of section 3.3,  we described the use of a single linear projection for all points can result to
stretched point clusters when they are not well aligned to the global principal components. We add an option to inflate potentially squeezed point clusters using six local rotations with  equally distanced angles in the interval $[0, \frac{\pi}{2}]$, followed by a scaling and the inverse rotation. This has no impact on the performance and only makes visualization more appealing. Here in Figure~\ref{fig:supp_inflation} we show a visual comparison on the \texttt{Fashion MNIST} dataset as an example.  

\section{KNN Accuracy Comparison}
Figure~\ref{fig:k_ablation} shows the KNN performance of methods on varying number of K on all datasets. As seen on all datasets h-NNE performs on par with the other methods on the whole range of k-neighbour values. 

\setlength{\belowcaptionskip}{0pt}
\setlength{\abovecaptionskip}{0pt}

\begin{figure*}[t]
\begin{center}
% \begin{overpic} 
% [width=\linewidth]
% {example-image-a}
% \end{overpic}
\includegraphics[width=\linewidth]{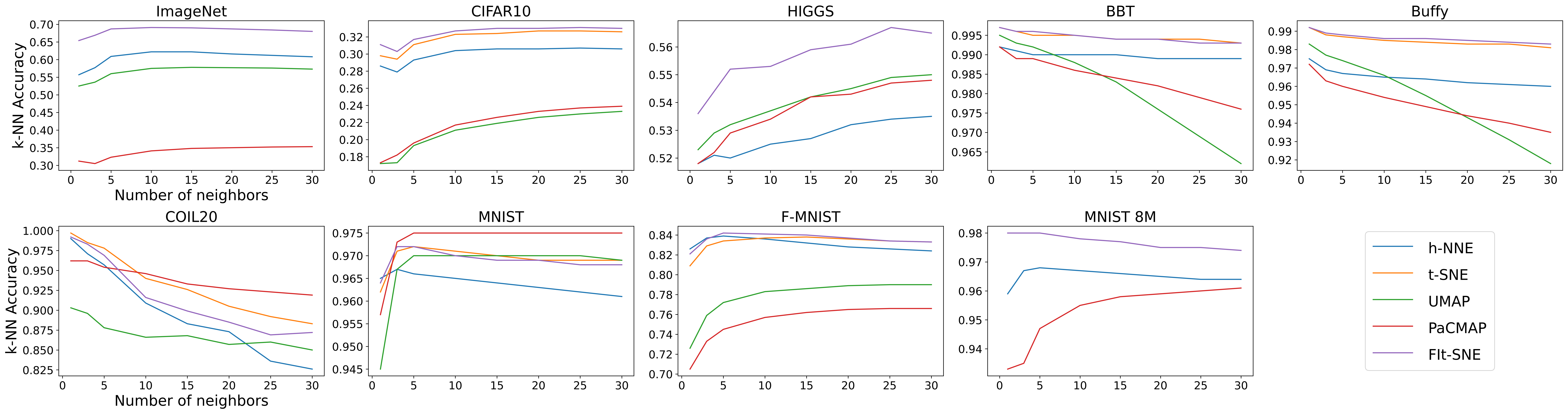}
\end{center}
\caption{
Local structure preservation: k-NN classifier accuracy on different datasets with increasing k neighbours.
}
\label{fig:k_ablation}
\end{figure*}

\begin{table}[t!]
\centering

\resizebox{\linewidth}{!}
%\resizebox{8.3cm}{!}
{
\begin{tabular}{c|c||c|c|c|c}
\toprule
data / GT clusters& Original-dim & FIt-SNE[18] & UMAP[20] & PaCMAP[31] & {\color{blue}{h-NNE}} \\
\midrule
 BBT / k=5& 89.7 & 55.3 & 54.2& 61.4 &87.6 \\
 \midrule
 Buffy / k=6& 76.1  &53.2 & 48.1& 41.5 & 74.7 \\
 \midrule
 ImageNet / k=1000& 74.3 &71.8 & 70.7& 63.9 & 70.9 \\
 \midrule
 MNIST8M / k=10& 61.3 & 55.5 & - & 57.7& 59.6 \\
\bottomrule
\end{tabular}}
\caption{NMI scores of k-means clustering}
\label{tab:nmi_scores} 
\end{table}

\section{Clustering Properties}
The method is built on the principle of grouping data points together in a hierarchical way which captures clustering properties.
To demonstrate, we cluster the large scale datasets before and after the projection with k-means and compare to their groundtruth clusters. Table~\ref{tab:nmi_scores} shows the NMI scores of clustering in the original high-dim feature space (Original-dim) and in the 2-dim projection space of different methods. As seen, in comparison h-NNE maintains a high NMI score that shows its ability to preserve the clusters better.

\renewcommand{\arraystretch}{0.65}
\setlength{\tabcolsep}{2pt}

\begin{figure*}[b]
    \centering
    \includegraphics[width=\linewidth]{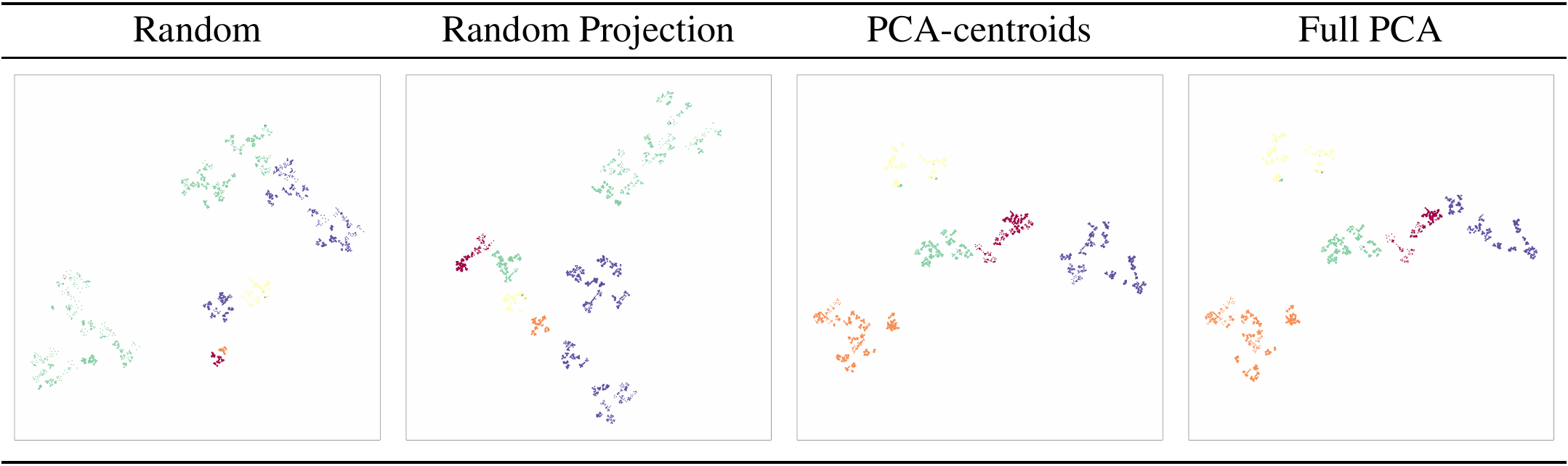}
%     \begin{tabular}{cccc}
%     \toprule
%     Random & Random Projection & PCA-centroids & Full PCA  \\\midrule
   
%     \includegraphics[width=0.2\linewidth]{fig/Images/rand_init/h-NNE_random.png}
%  &
%          \includegraphics[width=0.2\linewidth]{fig/Images/rand_init/h-NNE_random_projection.png}
%   &
%     \includegraphics[width=0.2\linewidth]{fig/Images/rand_init/h-NNE_pca.png}
% &
%          \includegraphics[width=0.2\linewidth]{fig/Images/rand_init/h-NNE_full_pca.png}
 
%   \\
%  \bottomrule
 
%  \end{tabular}
    \caption{Impact of preliminary projection - Random initialization versus PCA on the BBT dataset } 
    \label{fig:rand_init}
\end{figure*}
\renewcommand{\arraystretch}{0.65}
\setlength{\tabcolsep}{2pt}

\begin{figure*}[b]
    \centering
    \includegraphics[width=\linewidth]{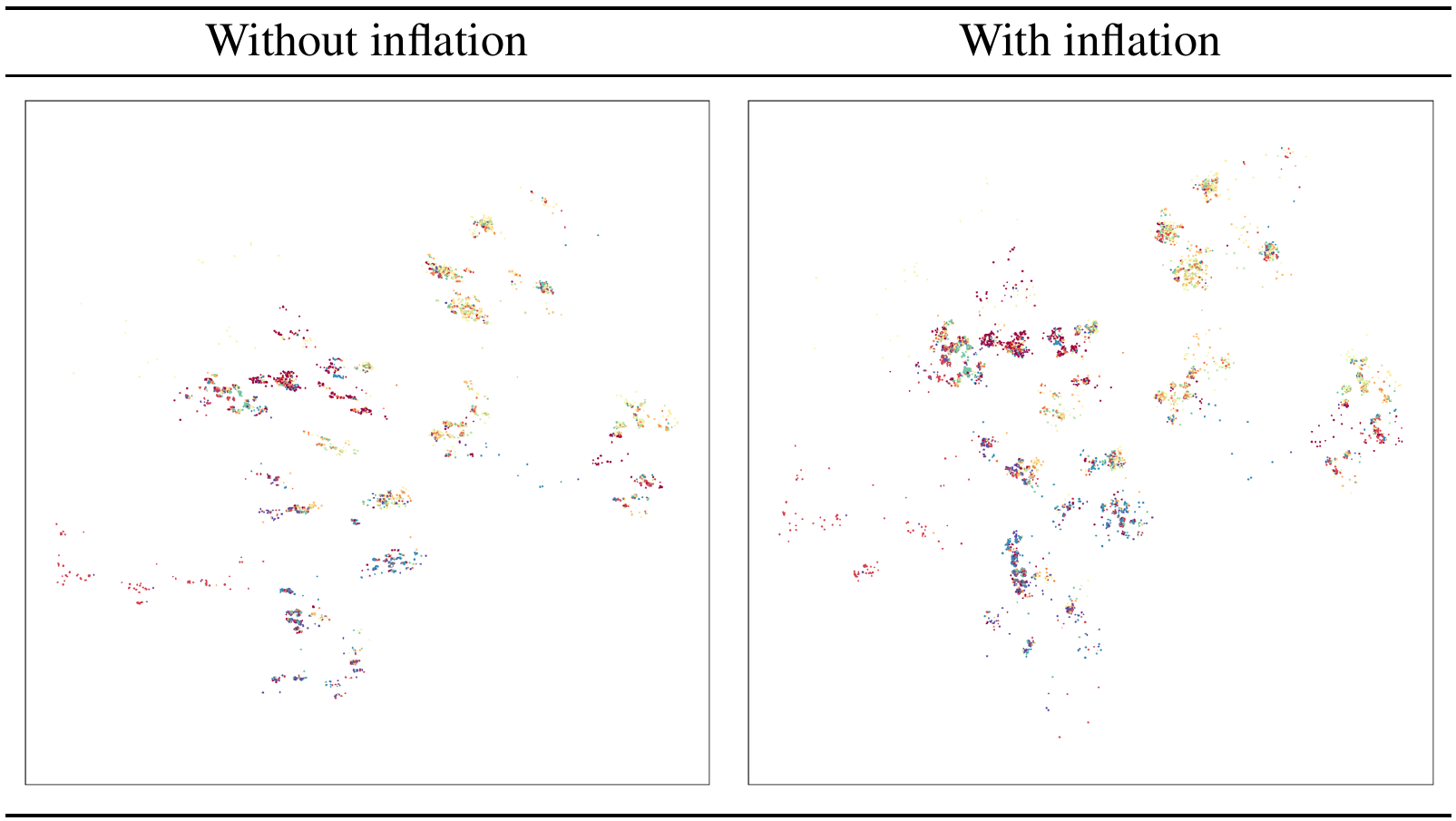}
%     \begin{tabular}{cc}
%     \toprule
%     Without inflation & With inflation  \\\midrule
%   \includegraphics[width=0.3\linewidth]{fig/Images/visual_comparison/cifar10/h-NNE_1.2noinflate.png}
%   &
   
%     \includegraphics[width=0.3\linewidth]{fig/Images/visual_comparison/cifar10/h-NNE_1.2.png}

%   \\
%  \bottomrule
 
%  \end{tabular}
    \caption{Impact of Point Cluster Inflation for visualization purposes on the Cifar10 dataset}
    \label{fig:supp_inflation}
\end{figure*}

\section{Visual comparison}
We end our supplementary with a visual comparison of the embeddings in two dimensions for a sample of real-world datasets which can be viewed in figure~\ref{fig:visual-comparison}. In each plot the colors are based on the labels of each dataset with the exception of Google News for which no labels are available. Also note that the Google News plot is missing for the case of t-SNE due to it not finishing the projection of 3 million points in a reasonable time interval.

\renewcommand{\arraystretch}{0.65}
\setlength{\tabcolsep}{2pt}

\begin{figure*}
    \centering
    \includegraphics[width=\linewidth]{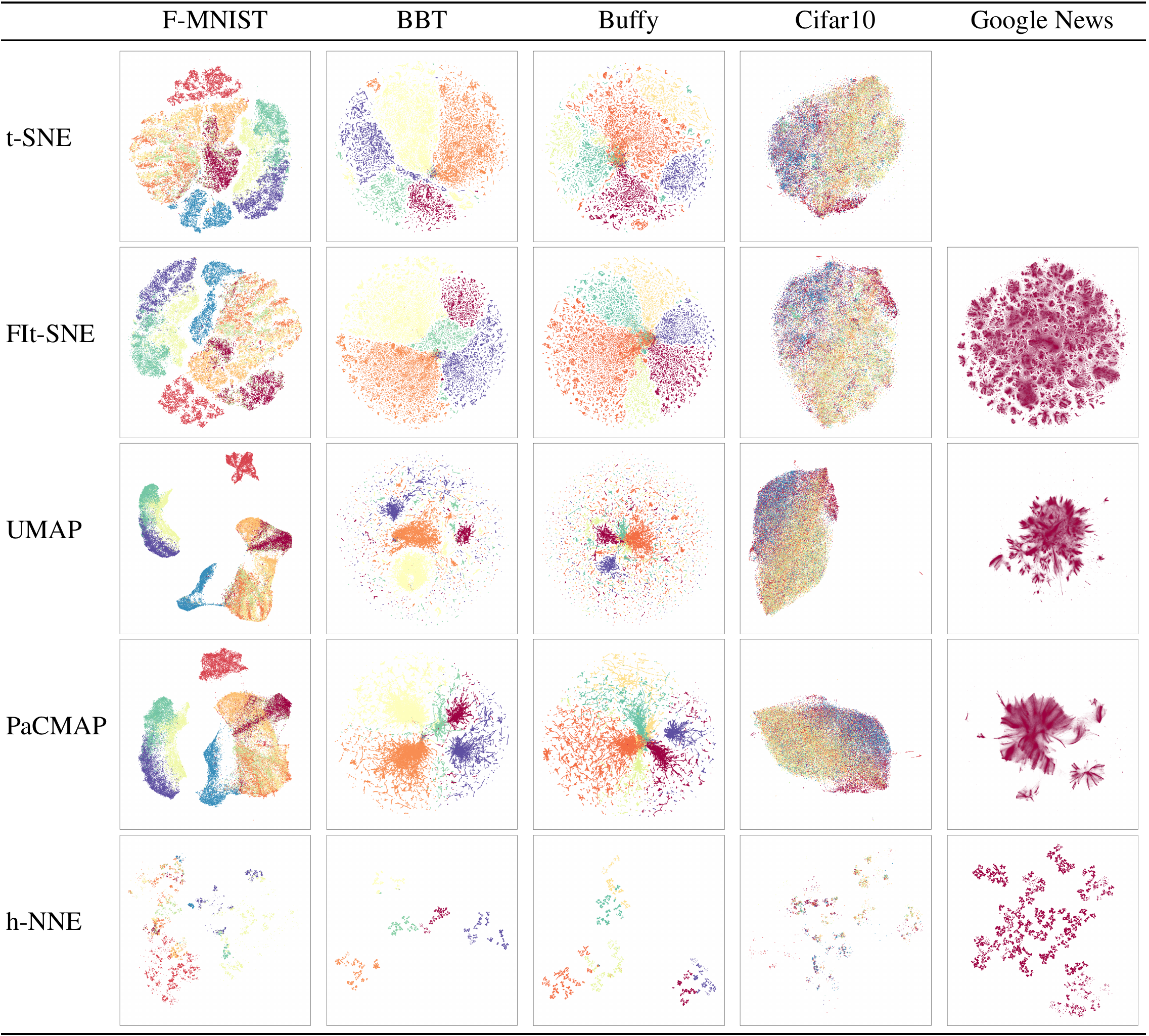}
    \caption{Visual comparison between h-NNE, t-SNE, FIt-SNE, UMAP and PaCMAP projections in 2D.} 
    \label{fig:visual-comparison}
\end{figure*}
\end{document}